\documentclass[sigconf,natbib=true,anonymous=false,review=false,screen=true]{acmart}

\acmSubmissionID{65}

\usepackage[inline]{enumitem}
\usepackage{acronym}
\usepackage{xkcdcolors}

\PassOptionsToPackage{dvipsnames}{xcolor}

\usepackage{dsfont}
\usepackage{natbib}
\usepackage[noend]{algorithmic}
\usepackage{algorithm}
\usepackage{bbm}
\usepackage{beramono}
\usepackage{bm}
\usepackage{booktabs}
\usepackage[skip=0pt]{caption}
\usepackage{cleveref}
\usepackage[T1]{fontenc}
\usepackage{graphicx}
\usepackage{hyperref}
\usepackage{listings}
\usepackage{multirow}
\usepackage{natbib}
\usepackage{subcaption}
\usepackage{tabularx}
\usepackage[dvipsnames]{xcolor}
\usepackage{numprint}
\usepackage[normalem]{ulem}
\usepackage{multirow}
\usepackage{float}
\usepackage{multicol, blindtext}
\usepackage{amsthm}
\usepackage[english]{babel}
\usepackage{amsthm}
\usepackage[T1]{fontenc}
\usepackage{mleftright}
\usepackage{tikz}
\usepackage{color}
\usepackage{graphicx}
\usepackage{caption}
\usepackage{subcaption}
\usetikzlibrary{backgrounds}
\usepackage{subcaption}
\usepackage{adjustbox}
\usepackage{mathtools}
\usepackage{appendix}
\usepackage{enumitem}

\copyrightyear{2024}
\acmYear{2024}
\setcopyright{rightsretained}
\acmConference[CIKM '24]{Proceedings of the 33rd ACM International Conference on Information and Knowledge Management}{October 21--25, 2024}{Boise, ID, USA}
\acmBooktitle{Proceedings of the 33rd ACM International Conference on Information and Knowledge Management (CIKM '24), October 21--25, 2024, Boise, ID, USA}
\acmDOI{10.1145/3627673.3679531}
\acmISBN{979-8-4007-0436-9/24/10}

\usepackage{etoolbox}
\makeatletter
\patchcmd{\maketitle}{\@copyrightpermission}{
  \begin{minipage}{0.3\columnwidth}
    \href{http://creativecommons.org/licenses/by/4.0/}{\includegraphics[width=0.90\textwidth]{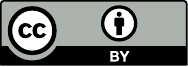}}
  \end{minipage}\hfill
  \begin{minipage}{0.7\columnwidth}
    \href{http://creativecommons.org/licenses/by/4.0/}{This work is licensed under a Creative Commons Attribution International 4.0 License.}
  \end{minipage}
  \vspace{5pt}
}{}{}
\makeatother

\acrodef{CF}{collaborative filtering}
\acrodef{LTR}{learning to rank}
\acrodef{NDCG}{normalized discounted cumulative gain}
\acrodef{DCG}{discounted cumulative gain}
\acrodef{VAE}{variational autoencoder}
\acrodef{VAE}{variational autoencoder}
\acrodef{ELBO}{evidence lower bound objective}
\acrodef{IPS}{inverse propensity scoring}
\acrodef{BPR}{bayesian personalized ranking}
\acrodef{MF}{matrix factorization}
\acrodef{MNAR}{missing-not-at-random}
\acrodef{ULTR}{unbiased learning-to-rank}
\acrodef{CLTR}{counterfactual learning to rank}
\acrodef{LOLN}{law of large numbers}
\acrodef{CRM}{counterfactual risk minimization}
\acrodef{IS}{importance sampling}
\acrodef{i.i.d}{independent and identically distributed}
\acrodef{CRM}{counterfactual risk minimization}
\acrodef{PL}{Plackett-Luce}
\acrodef{CTR}{click through rate}
\acrodef{SEA}{safe exploration algorithm}
\acrodef{GENSPEC}{generalization and specialization }
\acrodef{VCRM}{variational counterfactual risk minimization}
\acrodef{SGD}{stochastic gradient descent}
\acrodef{DR}{doubly robust}
\acrodef{DM}{direct method}
\acrodef{ERC}{exposure ratio clipping}
\acrodef{PRPO}{proximal ranking policy optimization}
\acrodef{PPO}{proximal policy optimization}
\acrodef{RL}{reinforcement learning}

\theoremstyle{plain}
\newtheorem{theorem}{Theorem}[section]

\newtheorem{lemma}[theorem]{Lemma}

\theoremstyle{definition}

\newtheorem{assumption}[theorem]{Assumption}
\theoremstyle{remark}

\newcommand{\headernodot}[1]{\vspace{1mm}\noindent\textbf{#1}}
\newcommand{\header}[1]{\headernodot{#1.}}

\allowdisplaybreaks

\author{Shashank Gupta}
\orcid{0000-0003-1291-7951}
\affiliation{%
  \institution{University of Amsterdam}
  \city{Amsterdam}
  \country{The Netherlands}
}
\email{s.gupta2@uva.nl}

\author{Harrie Oosterhuis}
\orcid{0000-0002-0458-9233}
\affiliation{%
  \institution{Radboud University}
  \city{Nijmegen}
  \country{The Netherlands}
}
\email{harrie.oosterhuis@ru.nl}

\author{Maarten de Rijke}
\orcid{0000-0002-1086-0202}
\affiliation{%
  \institution{University of Amsterdam}
  \city{Amsterdam}
  \country{The Netherlands}
}
\email{m.derijke@uva.nl}

\title[Practical and Robust Safety Guarantees for Advanced Counterfactual Learning to Rank]{Practical and Robust Safety Guarantees for\\ Advanced Counterfactual Learning to Rank}

\allowdisplaybreaks

\begin{document}

\begin{abstract}

\Ac{CLTR} can be risky and, in various circumstances, can produce sub-optimal models that hurt performance when deployed.
Safe \ac{CLTR} was introduced to mitigate these risks when using \acl{IPS} to correct for position bias.
However, the existing safety measure for \ac{CLTR} is not applicable to state-of-the-art \ac{CLTR} methods, cannot handle trust bias, and relies on specific assumptions about user behavior.

Our contributions are two-fold.
First, we generalize the existing safe \ac{CLTR} approach to make it applicable to state-of-the-art \acl{DR} \ac{CLTR} and trust bias.
Second, we propose a novel approach, \acfi{PRPO}, that provides safety in deployment without  assumptions about user behavior.
\ac{PRPO} removes incentives for learning ranking behavior that is too dissimilar to a safe ranking model.
Thereby, \ac{PRPO} imposes a limit on how much learned models can degrade performance metrics, \emph{without} relying on any specific user assumptions.
Our experiments show that both our novel safe \acl{DR} method and \ac{PRPO}  provide higher performance than the existing safe \acl{IPS} approach.
However, in unexpected circumstances, the safe \acl{DR} approach can become unsafe and bring detrimental performance.
In contrast, \ac{PRPO} always maintains safety, even in maximally adversarial situations.
By avoiding assumptions, \ac{PRPO} is the first method with \emph{unconditional} safety in deployment that translates to robust safety for real-world applications.
\vspace{-0.25\baselineskip}
\end{abstract}




\maketitle

\acresetall

% !TEX root = ../2024-crm-ppo.tex

\section{Introduction}

\Ac{CLTR}~\cite{joachims2017unbiased, wang2016learning, oosterhuis2020learning,gupta2024unbiased} concerns the optimization of ranking systems based on user interaction data using \ac{LTR} methods~\cite{liu2009learning}.
A main advantage of \ac{CLTR} is that it does not require manual relevance labels, which are costly to produce~\citep{qin2010letor,chapelle2011yahoo} and often do not align with actual user preferences~\cite{sanderson2010test}.
Nevertheless, \ac{CLTR} also brings significant challenges since user interactions only provide a heavily biased form of implicit feedback~\citep{gupta2024unbiased}.
User clicks are affected by many different factors, for example, the position at which an item is displayed  in a ranking~\citep{wang2018position, craswell2008experimental}.
Thus, click frequencies provide a biased indication of relevance, that is often more representative of how an item was displayed than actual user preferences~\citep{agarwal2019addressing, wang2016learning}.

To correct for this bias, early \ac{CLTR} applied \ac{IPS}, which weights clicks inversely to the estimated effect of position bias~\cite{joachims2017unbiased, wang2016learning}.
Later work expanded this approach to correct for other forms of bias, e.g., item-selection bias~\citep{oosterhuis2020policy, ovaisi2020correcting} and trust bias~\citep{agarwal2019addressing, vardasbi2020inverse}, and more advanced \ac{DR} estimation~\citep{oosterhuis2022doubly}.
Using these methods, standard \ac{CLTR} aims to create an unbiased estimate of relevance (or user preference) from click frequencies.
In other words, their goal is to output an estimate per document with an expected value that is equal to their relevance.

However, unbiased estimates of \ac{CLTR} have their limitations.
Firstly, they assume a model of user behavior and require an accurate estimate of this model.
If the assumed model is incorrect~\citep{vardasbi2020inverse, oosterhuis2020policy} or its estimated parameters are inaccurate~\citep{joachims2017unbiased, oosterhuis2022doubly}, then their unbiasedness is not guaranteed.
Secondly, even when unbiased, the estimates are subject to variance~\cite{oosterhuis2022reaching}.
As a result, the actual estimated values are often erroneous, especially when the available data is sparse~\citep{oosterhuis2022doubly, gupta2023safe}.
Accordingly, unbiased \ac{CLTR} does not guarantee that the ranking models it produces have optimal performance~\cite{oosterhuis2022reaching}.

\header{Safe \acl{CLTR}}
There are risks involved in applying \ac{CLTR} in practice. 
In particular, there is a substantial risk that a learned ranking model is deployed that degrades performance compared to the previous production system~\citep{gupta2023safe,oosterhuis2021robust,jagerman2020safe}.
This can have negative consequences to important business metrics, making \ac{CLTR} less attractive to practitioners.
To remedy this issue, a \emph{safe} \ac{CLTR} approach was proposed by \citet{gupta2023safe}.
Their approach builds on IPS-based \ac{CLTR} and adds exposure-based risk regularization, which keeps the learned model from deviating too much from a given safe model.
Thereby, under the assumption of a position-biased user model, the safe \ac{CLTR} approach can guarantee an upper bound on the probability of the model being worse than the safe model.

\header{Limitations of the current safe \ac{CLTR} method}
Whilst safe \ac{CLTR} is an important contribution to the field, it has two severe limitations -- both are addressed by this work.
Firstly, the existing approach is only applicable to \ac{IPS} estimation, which is no longer the state-of-the-art in the field~\citep{gupta2024unbiased}, and it assumes a rank-based position bias model~\citep{craswell2008experimental, wang2016learning}, the most basic user behavior model in the field.
Secondly, because its guarantees rely on assumptions about user behavior, it can only provide a conditional notion of safety.
Moreover, since user behavior can be extremely heterogeneous, it is unclear whether a practitioner could even determine whether the safety guarantees would apply to their application.

% \header{Safe \acl{DR} \ac{CLTR}}
\header{Main contributions}
Our first contribution addresses the mismatch between the existing safe \ac{CLTR} approach and recent advances in \ac{CLTR}.
We propose a novel generalization of the exposure-based regularization term that provides safety guarantees for both \ac{IPS} and \ac{DR} estimation, also under more complex models of user behavior that cover both position and trust bias.
Our experimental results show that our novel method reaches higher levels of performance significantly faster, while avoiding any notable decreases of performance.
This is especially beneficial since \ac{DR} is known to have detrimental performance when very little data is available~\citep{oosterhuis2022doubly}.

Our second contribution provides an unconditional notion of safety. We take inspiration from advances in \ac{RL}~\cite{wang2020truly,schulman2017proximal,liu2019neural,wang2019trust,queeney2021generalized} and propose the novel \acfi{PRPO} method.
\ac{PRPO} removes incentives for \ac{LTR} methods to rank documents too much higher than a given safe ranking model would.
Thereby, \ac{PRPO} imposes a limit on the performance difference between a learned model and a safe model, in terms of standard ranking metrics.
Importantly, \ac{PRPO} is easily applicable to \emph{any} gradient-descent-based \ac{LTR} method, and makes \emph{no assumptions} about user behavior.
In our experiments, \ac{PRPO} prevents any notable decrease in performance even under extremely adversarial circumstances, where other methods fail.
Therefore, we believe \ac{PRPO} is the first \emph{unconditionally} safe \ac{LTR} method. % in the field.

Together, our contributions bring important advances to the theory of safe \ac{CLTR}, by proposing a significant generalization of the existing approach with theoretical guarantees, and the practical appeal of \ac{CLTR}, with the first robustly safe \ac{LTR} method: \ac{PRPO}.
All source code to reproduce our experimental results is available at: 
\url{https://github.com/shashankg7/cikm-safeultr}.
%\href{https://anonymous.4open.science/r/cikm-safeultr-3E6C}{cikm24 safe cltr}.

\vspace{-1.1mm}

% !TEX root = ../2024-crm-ppo.tex

\section{Related Work}
\label{sec:relatedwork}
\textbf{\Acl{CLTR}.}
\citet{joachims2017unbiased} introduced the first method for \ac{CLTR}, a \ac{LTR} specific adaptation of \ac{IPS} from the bandit literature~\cite{swaminathan2015batch,joachims2016counterfactual,saito2021counterfactual,gupta2024optimal,gupta2023deep, gupta2024first} to correct for position bias.
They weight each user interaction according to the inverse of its examination probability, i.e., its inverse propensity, during learning to correct for the position bias in the logged data. 
This weighting will remove the effect of position bias from the final ranking policy.
\citet{oosterhuis2020policy} extended this method for the top-$K$ ranking setting with item-selection bias, where any item placed outside the top-$K$ positions gets zero exposure probability, i.e., an extreme form of position bias. 
They proposed a policy-aware propensity estimator, where the propensity weights used in \ac{IPS} are conditioned on the logging policy used to collect the data. 

\citet{agarwal2019addressing} introduced an extension of \ac{IPS}, known as Bayes-IPS, to correct for \textit{trust bias}, an extension of position-bias, with false-positive clicks at the higher ranks, because of the users' trust in the search engine. 
\citet{vardasbi2020inverse} proved that Bayes-IPS cannot correct for trust bias and introduced an affine-correction method and unbiased estimator. 
\citet{oosterhuis2021unifying} combined the affine-correction with a policy-aware propensity estimator to correct for trust bias and item-selection bias simultaneously.
Recently, \citet{oosterhuis2022doubly} introduced a \ac{DR}-estimator for \ac{CLTR}, which combines the existing \ac{IPS}-estimator with a regression model to overcome some of the challenges with the \ac{IPS}-estimator. 
The proposed \ac{DR}-estimator corrects for item-selection and trust biases, with lower variance and improved sample complexity. 

\header{Safe policy learning from user interactions}
In the context of offline evaluation for contextual bandits, \citet{thomas2015high} introduced a high-confidence off-policy evaluation framework.
A confidence interval is defined around the empirical off-policy estimates, and there is a high probability that the \textit{true} utility can be found in the interval. 
\citet{jagerman2020safe} extended this framework for safe deployment in the contextual bandit learning setup. 
The authors introduce a \ac{SEA} method that selects with high confidence between a safe behavior policy and the newly learned policy. 
In the context of \ac{LTR}, \citet{oosterhuis2021robust} introduced the \ac{GENSPEC} method, which safely selects between a feature-based and tabular \ac{LTR} model. 
For off-policy learning, \citet{swaminathan2015batch} introduced a \ac{CRM} framework for the contextual bandit setup. 
They modify the \ac{IPS} objective for bandits to include a regularization term, which explicitly controls for the variance of the \ac{IPS}-estimator during learning, thereby overcoming some of the problems with the high-variance of \ac{IPS}.
\citet{wu2018variance} extended the \ac{CRM} framework by using a \textit{risk} regularization, which penalizes mismatches in the action probabilities under the new policy and the behavior policy. 
\citet{gupta2023safe} made this general safe deployment framework effective in the \ac{LTR} setting. 
They proposed an exposure-based risk regularization method where the difference in the document exposure distribution under the new and logging policies is penalized. 
When click data is limited, risk regularization ensures that the performance of the new policy is similar to the logging policy, ensuring safety. 

To the best of our knowledge, the method proposed by~\citet{gupta2023safe} is the only method for safe policy learning in the \ac{LTR} setting. 
While it guarantees safe ranking policy optimization, it has two main limitations:
\begin{enumerate*}[label=(\roman*)]
        \item It is only applicable to the \ac{IPS} estimator; and  
        \item under the position-based click model assumption, the most basic click model in the \ac{CLTR} literature~\citep{joachims2017unbiased, oosterhuis2020learning, gupta2024unbiased}. 
\end{enumerate*} 

\header{Proximal policy optimization} 
In the broader context of \ac{RL}, \acfi{PPO} was introduced as a policy gradient method for training \ac{RL} agents to maximize long-term rewards~\cite{wang2020truly,schulman2017proximal,liu2019neural,wang2019trust,queeney2021generalized}.
\Ac{PPO} clips the importance sampling ratio of action probability under the new policy and the current behavior policy, and thereby, it prevents the new policy to deviate from the behavior policy by more than a certain margin.
\ac{PPO} is not directly applicable to \ac{LTR}, for the same reasons that the \ac{CRM} framework is not: the combinatorial action space of \ac{LTR} leads to extremely small propensities that \ac{PPO} cannot effectively manage~\cite{gupta2023safe}.
%In the context of \ac{LTR}, \ac{PPO} is not directly applicable since the actions in the ranking context are combinatorial, resulting in extremely small action propensities ($\pi(A \mid X)$), resulting in a high-variance and unstable estimator~\cite{gupta2023safe}.

% !TEX root = ../2024-crm-ppo.tex
\vspace*{-0.43cm}
\section{Background}
\subsection{Learning to rank}
The goal in \ac{LTR} is to find a ranking policy ($\pi$) that optimizes a given ranking metric~\cite{liu2009learning}. 
Formally, given a set of documents ($D$), a distribution of queries $Q$, and the true relevance function ($P(R=1 \mid d)$), \ac{LTR} aims to maximize the following utility function:
\begin{equation}
    U(\pi) =  \sum_{q \in Q} P(q \mid Q) \sum_{d \in D} \omega(d \mid \pi) \; P(R=1 \mid d), \label{true-utility}
\end{equation}
where $\omega(d \mid \pi)$ is the weight of the document for a given policy $\pi$. The weight can be set accordingly to optimize for a given ranking objective, for example, setting the weight to:
\begin{equation}
    \omega_{\text{DCG}}(d \mid q, \pi) = \mathbb{E}_{y \sim \pi( \cdot \mid q)} \mleft[ (\log_2(\textrm{rank}(d \mid y) + 1))^{-1} \mright], \label{rho}
\end{equation}
optimizes \ac{DCG}~\citep{jarvelin2002cumulated}.
For this paper, we aim to optimize the expected number of clicks, so we set the weight accordingly~\cite{oosterhuis2022doubly,gupta2023safe,yadav2021policy}. 
%We note that whilst our utility function is defined for a single query context, it can easily be extended by placing it inside an expectation over a distribution of queries~\cite{oosterhuis2022doubly}. % Harrie: I changed this because the later IPS and DR are introduced as expectations over queries!

\subsection{Assumptions about user click behavior}
The optimization of the true utility function (Eq.~\ref{true-utility}) requires access to the document relevance ($P(R=1 \mid d)$). 
In the \ac{CLTR} setting, the relevances of documents are not available, and instead, click interaction data is used to estimate them~\cite{joachims2017unbiased, wang2016learning, oosterhuis2020learning}.
However, naively using clicks to optimize a ranking system can lead to sub-optimal ranking policies, as clicks are a biased indicator of relevance~\cite{chuklin-click-2015, joachims2002optimizing, joachims2017unbiased, craswell2008experimental}.
\ac{CLTR} work with theoretical guarantees starts by assuming a model of user behavior.
The earliest \ac{CLTR} works~\citep{joachims2017unbiased, wang2016learning} assume a basic model originally proposed by \citet{craswell2008experimental}: 
\begin{assumption}[\emph{The rank-based position bias model}]
    \label{assumption:positionbias}
     The probability of a click on document $d$ at position $k$ is the product of the rank-based examination probability and document relevance:
    \begin{equation}    
    P(C = 1 \mid d, k) = P(E=1 \mid k) P(R=1 \mid d) = \alpha_k P(R=1 \mid d).
    \end{equation}
\end{assumption}

\noindent%
Later work has proposed more complex user models to build on~\citep{gupta2024unbiased}.
Relevant to our work is the model proposed by \citet{agarwal2019addressing}, and its re-formulation by \citet{vardasbi2020inverse}; it is a generalization of the above model to include a form of trust bias:

\begin{assumption}[\emph{The trust bias model}]
\label{assumption:trustbias}
The probability of a click on document $d$ at position $k$ is an affine transformation of the relevance probability of $d$ in the form:
\begin{equation}
        P(C = 1 \mid d, k) = \alpha_k P(R=1 \mid d) + \beta_k,  \label{affine-click-model}
\end{equation}
where $\forall k, \alpha_k \in [0,1] \land \beta_k\in [0,1] \land (\alpha_k + \beta_k) \in [0,1]$.
\end{assumption}

\noindent%
Whilst it is named after trust bias, this model actually captures three forms of bias that were traditionally categorized separately: rank-based position bias, item-selection bias, and trust bias.
Position bias was originally approached as the probability that a user would examine an item, which would decrease at lower positions in the ranking~\citep{wang2018position, craswell2008experimental, joachims2017unbiased, wang2016learning}.
In the trust bias model, this effect can be captured by decreasing $\alpha_k + \beta_k$ as $k$ increases.
Additionally, with $\forall k, \beta_k = 0$, the trust bias model is equivalent to the rank-based position bias model.
Item-selection bias refers to users being unable to see documents outside a top-$K$, where they receive zero probability of being examined or interacted with~\citep{oosterhuis2020policy}.
This can be captured by the trust bias model by setting $\alpha_k + \beta_k = 0$ when $k > K$.
Lastly, the key characteristic of trust bias is that users are more likely to click on non-relevant items when they are near the top of the ranking~\citep{agarwal2019addressing}.
This can be captured by the model by making $\beta_k$ larger as $k$ decreases~\citep{vardasbi2020inverse}.
Thereby, the trust bias model is in fact a generalization of most of the user models assumed by earlier work~\citep{oosterhuis2021unifying, gupta2024unbiased}. The following works all assume models that fit Assumption~\ref{assumption:trustbias}: \citep{vardasbi2020inverse, agarwal2019addressing, oosterhuis2021unifying, oosterhuis2020policy, wang2016learning, wang2018position, oosterhuis2022doubly, oosterhuis2021robust, gupta2023safe, agarwal2019general, ovaisi2020correcting}.

\subsection{Counterfactual learning to rank}
This section details the \emph{policy-aware} \acfi{IPS} estimator proposed by \citet{oosterhuis2020policy} and the \acfi{DR} estimator by \citet{oosterhuis2022doubly}.

First, let $\mathcal{D}$ be a set of logged interaction data:
%
%\begin{equation}
$
    \mathcal{D} = \big\{q_i, y_i, c_i \big\}^N_{i=1}$,
%    \label{logs}
%\end{equation}
%
where each of the $N$ interactions consists of a query $q_i$, a displayed ranking $y_i$, and click feedback $c_i(d) \in \{0,1\}$ that indicates whether the user clicked on the document $d$ or not.
Both policies use propensities that are the expected $\alpha$ values for each document:
\begin{equation}
        \rho_{0}(d \mid q_i, \pi_0) =  \mathbb{E}_{y \sim \pi_{0}(q_i)} \big[ \alpha_{k(d)}  \big] = \rho_{i,0}(d).
    \label{policy-aware-exposure}
\end{equation}
Similarly, to keep our notation short, we also use $\omega(d \mid q_i, \pi) = \omega_i(d)$.
Next, the policy-aware \ac{IPS} estimator is defined as:
\begin{equation}
    \hat{U}_{\text{IPS}}(\pi) = \frac{1}{N} \sum_{i=1}^{N} \sum_{d \in D}  \frac{\omega_i(d)}{\rho_{i,0}(d)} c_i(d).
    \label{cltr-obj-ips-positionbias}
\end{equation}
\citet{oosterhuis2020policy} prove that under the rank-based position bias model (Assumption~\ref{assumption:positionbias}) and when $\forall (i,d),  \rho_{i,0}(d) > 0$, this estimator is unbiased: $\mathbb{E}[\hat{U}_{\text{IPS}}(\pi)] = U(\pi)$.

The \ac{DR} estimator improves over the policy-aware \ac{IPS} estimator in terms of assuming the more general trust bias model (Assumption~\ref{assumption:trustbias}) and having lower variance.
\citet{oosterhuis2022doubly} proposes the usage of the following $\omega$ values for the policy $\pi$:
\begin{equation}
\omega(d \mid q_i, \pi) = \mathbb{E}_{y \sim \pi(q_i)} \big[ \alpha_{k(d)}  + \beta_{k(d)} \big] = \omega_i(d),
    \label{eq:omega}
\end{equation}
since with these values $U$ (Eq.~\ref{true-utility}) becomes the number of expected clicks on relevant items under the trust bias model; $U = (\alpha_{k}  + \beta_{k})P(R=1 \mid d,q) = P(C = 1, R = 1 \mid k,d,q)$.
We follow this approach and define the $\omega$ values for the logging policy $\pi_{0}$ as:
\begin{equation}
    \omega_{0}(d \mid q_i, \pi_{0}) = \mathbb{E}_{y \sim \pi_{0}(q_i)} \big[ \alpha_{k(d)}  + \beta_{k(d)} \big] = \omega_{i,0}(d).
    \label{eq:omega_logging}
\end{equation}
The \ac{DR} estimator uses predicted relevances in its estimation, i.e., using predictions from a regression model.
Let $\hat{R}_{i}(d) \approx P(R = 1 \mid d, q_i)$ indicate a predicted relevance; then the utility according to these predictions is:
\vspace{-\baselineskip}
\begin{equation}
    \hat{U}_{\text{DM}}(\pi) = \frac{1}{N} \sum_{i=1}^{N} \sum_{d \in D} \omega_i(d) \hat{R}_{i}(d).
\end{equation}
The \ac{DR} estimator starts with this predicted utility and adds an \ac{IPS}-based correction to remove its bias:
\begin{align}
    &\hat{U}_{\text{DR}}(\pi) = {} \label{cltr-obj-dr} \\[-2ex]%[-2.5ex]
    &\qquad \hat{U}_{\text{DM}}(\pi) +  \frac{1}{N} \sum_{i=1}^{N} \sum_{d \in D}  \frac{\omega_{i}(d)}{\rho_{i,0}(d)} \left(c_i(d) - \alpha_{k_i(d)}\hat{R}_{i}(d) -  \beta_{k_i(d)} \right). 
    \nonumber
\end{align}
Thereby, the corrections of the \ac{IPS} part of the \ac{DR} estimator will be smaller if the predicted relevances are more accurate.
\citet{oosterhuis2022doubly} proves that under the assumption of the trust bias model (Assumption~\ref{assumption:trustbias}), the \ac{DR} estimator is unbiased when $\forall (i,d), \rho_{i,0}(d) > 0 \lor \hat{R}_{i}(d) = P(R = 1 \mid d, q_i)$ and has less variance if $0 \leq \hat{R}_{i}(d) \leq 2P(R = 1 \mid d, q_i)$.
They also show that the \ac{DR} estimator needs less data to reach the same level of ranking performance as \ac{IPS}, with especially large improvements when applied to top-$K$ rankings~\citep{oosterhuis2022doubly}.

\subsection{Safety in counterfactual learning to rank}
\Ac{IPS}-based \ac{CLTR} methods, despite their unbiasedness and consistency, suffer from the problem of high-variance~\cite{gupta2024unbiased,oosterhuis2022doubly,joachims2017unbiased}.
 Specifically, if the logged click data is limited, training an \ac{IPS}-based method can lead to an unreliable and unsafe ranking policy~\cite{gupta2023safe}.
 The problem of \textit{safe} policy learning is well-studied in the bandit literature~\citep{thomas2015high, jagerman2020safe, swaminathan2015batch, wu2018variance}. \citet{swaminathan2015batch} proposed the first risk-aware off-policy learning method for bandits, with their risk term quantified as the variance of the \ac{IPS}-estimator. 
\citet{wu2018variance} proposed an alternative method for risk-aware off-policy learning, where the risk is quantified using a Renyi divergence between the action distribution of the new policy and the logging policy~\citep{renyi1961measures}.
Thus, both consider it a risk for the new policy to be too dissimilar to the logging policy, which is presumed safe.
Whilst effective at standard bandit problems, these risk-aware methods are not effective for ranking tasks due to their enormous combinatorial action spaces and correspondingly small propensities.

As a solution for \ac{CLTR}, \citet{gupta2023safe} introduced a risk-aware \ac{CLTR} approach that uses divergence based on the exposure distributions of policies.
They first introduce normalized propensities: $\rho'\!(d) = \rho / Z$, with a normalization factor $Z$ based on $K$:
\begin{equation}
        Z = \!\! \sum_{d \in D}^{}\! \rho(d)  = \!\! \sum_{d \in D} \! \mathbb{E}_{y \sim \pi } \big[ \alpha_{k(d)}  \big]  \! 
        % = \mathbb{E}_{y \sim \pi} \Big[  \sum_{d \in D}  P\big(E=1 \mid \text{rank}(d \mid y)\big)  \Big]    \\
         = \mathbb{E}_{y \sim \pi} \left[  \sum_{k=1}^K  \alpha_{k(d)}  \right] \!=  \!\! \sum_{k=1}^K  \alpha_k.
    \label{total-exposure}
\end{equation}
Since $\rho'\!(d) \in [0,1]$ and $\sum_d \rho'\!(d) = 1$, they can be treated as a probability distribution that indicates how exposure is spread over documents.
\citet{gupta2023safe} use Renyi divergence to quantify how dissimilar the new policy is from the logging policy:
\begin{equation}
    d_2(\rho \,\Vert\, \rho_0) = \mathbb{E}_{q} \left[ \sum_{d} \mleft(\frac{\rho'(d)}{\rho_{0}'(d)}\mright)^2 \rho_{0}'(d) \right],
    \label{eq:actionbaseddiv}
\end{equation} 
with the corresponding empirical estimate based on the log data ($\mathcal{D}$) defined as:
 \begin{equation}
    \hat{d}_2(\rho \,\Vert\, \rho_0) = \frac{1}{N} \sum_{i=1}^{N} \sum_{d} \mleft(\frac{\rho'_{i}(d)}{\rho_{i,0}'(d)}\mright)^2 \rho_{i,0}'(d).
    \label{eq:actionbaseddiv}
\end{equation} 
Based on this divergence term, they propose the following risk-aware \ac{CLTR} objective, with parameter $\delta$:
\begin{equation}
    \max_{\pi}  \hat{U}_{\text{IPS}}(\pi) -  \sqrt{ \frac{Z}{N}  \Big(\frac{1-\delta}{\delta}\Big) \hat{d}_2(\rho \,\Vert\, \rho_0)}.
\label{objgenbound}
\end{equation}
Thereby, the existing safe \ac{CLTR} approach penalizes the optimization procedure from learning ranking behavior that is too dissimilar from the logging policy in terms of the distribution of exposure.
The weight of this penalty decreases as the number of datapoints $N$ increases, thus it maintains the same point of convergence as standard \ac{IPS}.
Yet, initially when little data is available and the effect of variance is the greatest, it forces the learned policy to be very similar to the safe logging policy.
\citet{gupta2023safe} prove that their objective bounds the real utility with a probability of $1-\delta$:
\begin{equation}
P\mleft( U(\pi) \geq \hat{U}_{\text{IPS}}(\pi)\! -  \sqrt{ \frac{Z}{N}  \Big(\frac{1-\delta}{\delta}\Big) d_2(\rho \,\Vert\, \rho_{0})\!}\mright) \geq 1 - \delta.
\end{equation}
However, their proof of safety relies on the rank-based position bias model (Assumption~\ref{assumption:positionbias}) and their approach is limited to the basic \ac{IPS} estimator for \ac{CLTR}.

\subsection{Proximal policy optimization}
In the more general \acf{RL} field, \acfi{PPO} was introduced as a method to restrict a new policy $\pi$ from deviating too much from a previously rolled-out policy $\pi_0$~\cite{schulman2015high, schulman2017proximal}.
In contrast with the earlier discussed methods, \ac{PPO} does not make use of a divergence term but uses a simple clipping operation in its optimization objective.
Let $s$ indicate a state, $a$ an action and $R$ a reward function, the \ac{PPO} loss is:
\begin{equation}
    U^{PPO}\!(s, a, \pi, \pi_0) = \mathbb{E} \mleft[ \min \mleft( \frac{\pi(a \,|\, s)}{\pi_0(a \,|\, s)} R(a \,|\, s) , g\big(\epsilon, R(a \,|\, s) \big)  \mright) \mright],
    \end{equation}
where $g$ creates a clipping threshold based on the sign of $R(a \,|\, s)$:
\begin{equation}
    g\mleft(\epsilon, R(a \,|\, s)  \mright) = 
\begin{cases}
    \left(1+\epsilon\right) R(a \,|\, s) & \text{if }R(a \,|\, s)  \geq 0,\\
    \left(1-\epsilon\right) R(a \,|\, s) & \text{otherwise}.
\end{cases}    
\end{equation}
The clipping operation removes incentives for the optimization to let $\pi$ deviate too much from $\pi_0$, since there are no further increases in $U^{PPO}$ when $\pi(a \,|\, s) > (1+ \epsilon)\pi_0(a \,|\, s)$ or  $\pi(a \,|\, s) < (1- \epsilon)\pi_0(a \,|\, s)$, depending on the sign of $R(a \,|\, s)$.
Similar to the previously discussed general methods, \ac{PPO} is not effective when directly applied to the \ac{CLTR} setting due to the combinatorial action space and corresponding extremely small propensities (for most $a$ and $s$: $\pi_0(a \,|\, s) \simeq 0$).

% !TEX root = ../2024-crm-ppo.tex

%\section{Extending Safety Guarantees to Advanced CLTR}
\section{Extending Safety to Advanced CLTR}
\label{sec:method:existingCLTR}
In this section, we introduce our first contribution: our extension of the safe \ac{CLTR} method to address trust bias and \ac{DR} estimation. 

\subsection{Method: Safe doubly-robust CLTR}
For the safe \ac{DR} \ac{CLTR} method, we extend the generalization bound from the existing \ac{IPS} estimator and position bias~\cite[Eq.~26]{gupta2023safe} to the \ac{DR} estimator and trust bias. 
    \label{sec:perfbound}
    % Using the upper bound on the variance of an \ac{CLTR} \ac{IPS} estimator that was proven in Theorem~\ref{var-theorm}, we can now introduce a generalization bound for the \ac{CLTR} estimator.
    \begin{theorem}
    \label{CLTR-bound}
        Given the true utility $U(\pi)$ (Eq.~\ref{true-utility}) and its exposure-based \ac{DR} estimate $\hat{U}_{\text{DR}}(\pi)$ (Eq.~\ref{cltr-obj-dr}) of the ranking policy $\pi$ with the logging policy $\pi_{0}$ and the metric weights $\omega$ and $\omega_{0}$ (Eq.~\ref{eq:omega} and \ref{eq:omega_logging}), assuming the trust bias click model (Assumption~\ref{assumption:trustbias}),
        the following generalization bound holds with probability $1 - \delta$:
    %
%    \begin{equation*}
%        P\mleft(\! U\!(\pi) \!\geq\! \underset{\hspace{0.8mm}\text{DR}\hspace{-0.8mm}}{\hat{U}}\!(\pi)
%        \! - \!
%        \mleft(\! 1 \! + \! \max_{k} \! \frac{ \beta_k}{\alpha_k} \! \mright)
%        \! \frac{1}{N} \! \sum_{i=1}^N \!
%        \sqrt{ \! \frac{2 \mathrm{Z}}{N} \!\Big(\! \frac{1-\delta}{\delta}\! \Big) d_2(\omega_i \Vert \omega_{i,0}) \! }\mright) \! \geq \! 1 - \delta.
%%    \label{variance}
%    \end{equation*}
    %
    %
     \begin{equation*}
        P\mleft( U(\pi) \! \geq \! \hat{U}_\text{DR}(\pi)
        \! - \!
        \mleft( 1 \! + \! \max_{k} \! \frac{ \beta_k}{\alpha_k} \mright)
        \sqrt{\frac{2 Z}{N} \! \Big( \frac{1-\delta}{\delta} \Big) d_2(\omega \Vert \omega_{0}) }\mright) \! \geq \! 1 \! - \! \delta.
%    \label{variance}
    \end{equation*}
    \end{theorem}
    \begin{proof}
        For a proof, we refer to the appendix (Theorem~\ref{sec:perfbound-proof}).
    \end{proof}
    
\noindent%
Given the novel generalization bound from Theorem~\ref{CLTR-bound}, we define the safe \ac{DR} \ac{CLTR} objective as follows:
\begin{equation}
    \max_{\pi}  \hat{U}_{\text{DR}}(\pi) - \mleft( 1 + \max_{k} \frac{ \beta_k}{\alpha_k} \mright) \sqrt{ \frac{2 Z}{N}  \Big(\frac{1-\delta}{\delta}\Big) \hat{d}_2(\omega \,\Vert\, \omega_{0})},
\label{dr-objgenbound}
\end{equation}
where $\hat{d}_2(\omega \,\Vert\, \omega_0)$ is defined analogously to Eq.~\ref{eq:actionbaseddiv}.
The objective optimizes the lower-bound on the true utility function, through a linear combination of the empirical \ac{DR} estimator ($\hat{U}_{\text{DR}}(\pi)$) and the empirical risk regularization term ($\hat{d}_2(\omega \,\Vert\, \omega_0)$). 
In a setting where click data is limited, our safe \ac{DR} objective will weight the risk regularization term higher, and as a result, the objective ensures that the new policy stays close to the safe logging policy. 
When a sufficiently high volume of click data is collected, and thus we have higher confidence in the \ac{DR} estimate, the objective falls back to its \ac{DR} objective counterpart.

For the choice of the ranking policy ($\pi$), we propose to optimize a stochastic ranking policy $\pi$ with a gradient descent-based method. 
For the gradient calculation, we refer to previous work~\cite{yadav2021policy,gupta2023safe,oosterhuis2022doubly}.

\header{Conditions for safe \ac{DR} \ac{CLTR}}
Finally, we note that besides the explicit assumption that user behavior follows the trust bias model (Assumption~\ref{assumption:trustbias}), there is also an important implicit assumption in this approach.
Namely, the approach assumes that the bias parameters (i.e., $\alpha$ and $\beta$) are known, a common assumption in the \ac{CLTR} literature~\citep{oosterhuis2022doubly, oosterhuis2020learning}.
However, in practice, either of these assumptions could not hold, i.e., user behavior could not follow the trust bias model, or a model's bias parameters could be wrongly estimated.
Additionally, in adversarial settings where clicks are intentionally misleading or incorrectly logged~\cite{radlinski2007addressing,liu2023black,castillo2011adversarial}, the user behavior assumptions do not hold, and, the generalization bound of our \ac{DR} \ac{CLTR} is not guaranteed to hold.
Thus, whilst it is an important advancement over the existing safe \ac{CLTR} method~\cite{gupta2023safe}, our approach is limited to only providing a \emph{conditional} form of safety.

%\vspace*{-1.5mm}
\section{Method: Proximal Ranking Policy Optimization (PRPO)}

Inspired by the limitations of the method introduced in Section~\ref{sec:method:existingCLTR} and the \ac{PPO} method from the \ac{RL} field (Section~\ref{sec:relatedwork}), we propose the first \emph{unconditionally} safe \ac{CLTR} method: \acfi{PRPO}.
Our novel \ac{PRPO} method is designed for practical safety by making \emph{no assumptions} about user behavior.
Thereby, it provides the most robust safety guarantees for \ac{CLTR} yet.

For safety, instead of relying on a high-confidence bound (e.g., Eq.~\ref{objgenbound} and~\ref{dr-objgenbound}), \ac{PRPO} guarantees safety by removing the incentive for the new policy to rank documents too much higher than the safe logging policy. 
This is achieved by directly clipping the ratio of the metric weights for a given query $q_i$ under the new policy $\omega_i(d)$, and the logging policy ($\omega_{i,0}(d)$), i.e., $\frac{\omega_{i}(d)}{\omega_{i,0}(d)}$ to be bounded in a fixed predefined range: $\mleft[\epsilon_{-}, \epsilon_{+}\mright]$. 
As a result, the \ac{PRPO} objective provides no incentive for the new policy to produce weights $\omega_i(d)$ outside of the range: $\epsilon_{-} \cdot \omega_{i,0}(d) \leq \omega_i(d) \leq \epsilon_{+} \cdot \omega_{i,0}(d)$.

Before defining the \ac{PRPO} objective, we first introduce a term $r(d|q)$ that represents an unbiased \ac{DR} relevance estimate, weighted by $\omega_{0}$, for a single document-query pair (cf.\ Eq.~\ref{cltr-obj-dr}):
\begin{equation}
\begin{split}
        \mbox{}\hspace{-0.2cm}&r(d | q)  = {}\\[-1ex]
        \mbox{}\hspace{-0.2cm}&\omega_{0}(d | q)\hat{R}(d | q) +   \frac{\omega_{0}(d | q)}{\rho_{0}(d | q)} \! \sum_{i \in \mathcal{D} : q_i = q \hspace{-0.7cm}}  \mleft(c_i(d)  - \alpha_{k_i(d)}\hat{R}(d | q)   -  \beta_{k_i(d)}\mright).\hspace{-0.1cm}\mbox{}
\end{split}        
\end{equation}
For the sake of brevity, we drop $\pi$ and $\pi_{0}$ from the notation when their corresponding value is clear from the context.
This enables us to reformulate the \ac{DR} estimator around the ratios between the metric weights $\omega$ and $\omega_0$ (cf.\ Eq.~\ref{cltr-obj-dr}):
\begin{equation}
    \hat{U}_{\text{DR}}(\pi) =  \sum_{q,d \in \mathcal{D}}  \frac{\omega(d\mid q)}{\omega_{0}(d \mid q)} r(d \mid q).
    \label{eq:dr_reformulate}
\end{equation}
Before defining the proposed \ac{PRPO} objective, we first define the following clipping function:
\begin{equation}
    f(x,\epsilon_{-}, \epsilon_{+}, r) = 
\begin{cases}
    \min(x, \epsilon_{+}) \cdot r  & r \geq 0,\\
    \max(x, \epsilon_{-}) \cdot r & \text{otherwise}.
\end{cases}  
\label{eq:clipping_function}  
\end{equation}
Given the reformulated \ac{DR} estimator (Eq.~\ref{eq:dr_reformulate}), and the clipping function (Eq.~\ref{eq:clipping_function}), the \ac{PRPO} objective can be defined as follows:
\begin{equation}
    \hat{U}_{\text{PRPO}}(\pi) =  \sum_{q,d \in \mathcal{D}}  f\mleft(\frac{\omega(d \mid q)}{\omega_{0}(d \mid q)}, \epsilon_{-}, \epsilon_{+}, r(d \mid q)\mright).
    \label{eq:prpo_obj}
\end{equation}
Figure~\ref{fig:prpo} visualizes the effect the clipping of \ac{PRPO} has on the optimization incentives.
We see how the clipped and unclipped weight ratios progress as documents are placed on different ranks.
The unclipped weights keep increasing as documents are moved to the top of the ranking, when $r>1$, or to the bottom, when $r<1$.
Consequently, optimization with unclipped weight ratios aims to place these documents at the absolute top or bottom positions.
Conversely, the clipped weights do not increase beyond their clipping threshold, which for most document is reached before being placed at the very top or bottom position.
As a result, optimization with clipped weight ratios will not push these documents beyond these points in the ranking.
For example, when $r>0$, we see that there is no incentive to place a document at higher than rank 6, if it was placed at rank 8 by the logging policy.
Similarly, placement higher than rank 4 leads to no gain if the original rank was 6, and higher than rank 3 leads to no improvement gain from an original rank of 4.
Vice versa, when $r<0$, each document has a rank, where placing it lower than that rank brings no increase in clipped weight ratio.
Importantly, this behavior only depends on the metric and the logging policy; \ac{PRPO} makes \emph{no further assumptions}.

\begin{figure}[!t]
    \begin{center}
    \includegraphics[scale=0.48]{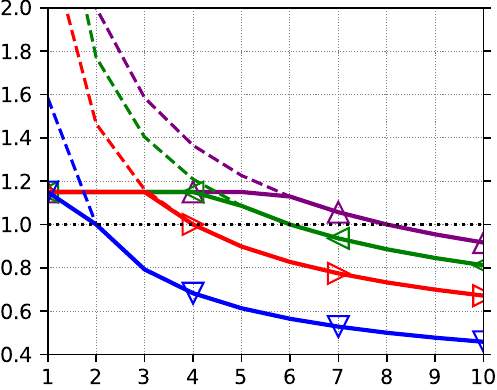}
    \includegraphics[scale=0.48]{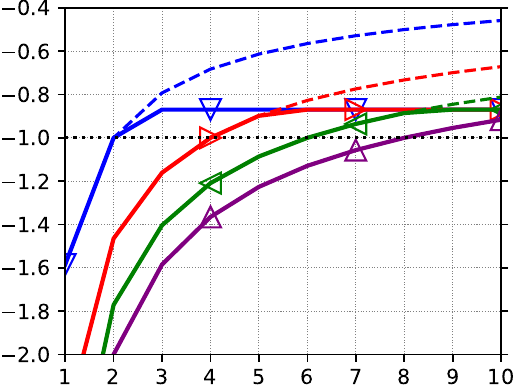}\\
    \includegraphics[width=0.99\columnwidth]{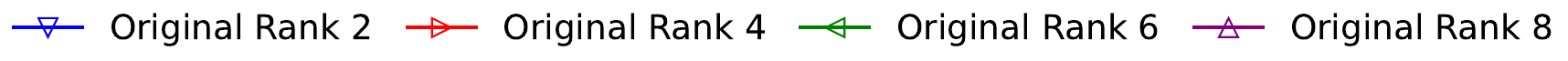}
    \end{center}
        \caption{
        Weight ratios in the clipped \ac{PRPO} objective (solid lines) and the unclipped counterparts (dashed lines), as documents are moved from four different original ranks.
        Left: positive relevance, $r=1$; right: negative relevance, $r=-1$;
        x-axis: new rank for document;
        y-axis: unclipped weight ratios (dashed lines), $r\cdot\omega_i(d)/\omega_{i,0}(d)$;
        and 
        clipped \ac{PRPO} weight ratios (solid lines),
        $f\mleft(\omega_i(d)/\omega_{i,0}(d), \epsilon_{-} = 1.15^{-1}, \epsilon_{+}= 1.15, r=\pm1\mright)$.
        DCG metric weights used: $\omega_i(d) = \log_2(\textmd{rank}(d \mid q_i, \pi) + 1)^{-1}$.
        }
        \label{fig:prpo}
\end{figure}

Whilst the clipping of \ac{PRPO} is intuitive, we can prove that it provides the following formal form of unconditional safety:
\begin{theorem}
\label{PRPO-proof}
Let $q$ be a query, $\omega$ be metric weights, $y_0$ be a logging policy ranking, and $y^*(\epsilon_{-},\epsilon_{+})$ be the ranking that optimizes the \ac{PRPO} objective in Eq.~\ref{eq:prpo_obj}.
Assume that $\forall d, \in \mathcal{D}, r(d \mid q) \not= 0$.
Then, 
%For any query $q_i$, any choice of metric weights $\omega$, and any logging policy ranking $y_0$, let $y^*(\epsilon_{-},\epsilon_{+})$ be the ranking that optimizes the \ac{PRPO} objective in Eq.~\ref{eq:prpo_obj} and $\forall q,d, \in \mathcal{D}, r(d \mid q) \not= 0$,
for any $\Delta \in \mathbb{R}_{\geq0}$, there exist values for $\epsilon_{-}$ and $\epsilon_{+}$ that guarantee that the difference between the utility of $y_0$ and $y^*(\epsilon_{-},\epsilon_{+})$ is bounded by $\Delta$:
\begin{equation}
\forall \Delta \! \in \mathbb{R}_{\geq0}, \exists \epsilon_{-} \!\!\in \mathbb{R}_{\geq0}, \epsilon_{+}\!\! \in \mathbb{R}_{\geq0};  \;| U(y_0) -  U(y^*(\epsilon_{-}, \epsilon_{+}))  | \leq \Delta.
\label{eq:prpotheorem}
\end{equation}
\end{theorem}
\begin{proof}
A proof is given in Appendix~\ref{sec:prpo-proof}.
% step 2: for each document there is a maximum displacement in ranks, because moving beyond $k^*_d$ is going to disrupt another document from going to its $k^*_{d'}$, since all $r_i(d) > 0$ 
% \\
% step 3: from maximum displacements you can always infer an upper bound on the change in metric, e.g., with precision@10 if only 2 items are moved into/out of the top 10 than it cannot change more than 20\%, (compute a worst case scenario)
\end{proof}
\vspace*{-3.5mm}
\header{Adaptive clipping}
Theorem~\ref{PRPO-proof} describes a very robust sense of safety, as it shows \ac{PRPO} can be used to prevent any given decrease in performance without assumptions.
However, it also reveals that this safety comes at a cost; \ac{PRPO} prevents both decreases and increases of performance.
This is very common in safety approaches, as there is a generally a tradeoff between risks and rewards~\citep{gupta2023safe}.
Existing safety methods, such as the safe \ac{CLTR} approach of Section~\ref{sec:method:existingCLTR}, generally, loosen their safety measures as more data becomes available, and the risk is expected to have decreased~\citep{thomas2015high}.

We propose a similar strategy for \ac{PRPO} through adaptive clipping, where the effect of clipping decreases as the number of datapoints $N$ increases.
Specifically, we suggest using a monotonically decreasing $\delta(N)$ function such that $\lim_{N \rightarrow \infty} \delta(N) = 0$.
The $\epsilon$ parameters can then be obtained through the following transformation: $\epsilon_{-} = \delta(N)$ and $\epsilon_{+} = \frac{1}{\delta(N)}$.
%
%\begin{equation}
%    \epsilon_{-} = \frac{1}{\delta(N)}, \qquad   \epsilon_{+} = \delta(N),
%    \label{eq:adapt_clip}
%\end{equation}
%
This leads to a clipping range of $[\delta(N), \frac{1}{\delta(N)}]$, and in the limit: $\lim_{N \rightarrow \infty}$, it becomes: $[0,\infty]$.
In other words, as more data is gathered, the effect of \ac{PRPO} clipping eventually disappears, and the original objective is recovered.
The exact choice of $\delta(N)$ determines how quickly this happens.

\header{Gradient ascent with PRPO and possible extensions}
Finally, we consider how the \ac{PRPO} objective should be optimized. This turns out to be very straightforward when we look at its gradient.
The clipping function $f$ (Eq.~\ref{eq:clipping_function}) has a simpler gradient involving an indicator function on whether $x$ is inside the bounded range:
\begin{equation}
\nabla_{x} f(x, \epsilon_{-}, \epsilon_{+}, r)
=  \mathds{1}\big[
(r > 0 \land x \leq \epsilon_{+})
\lor
(r < 0 \land x \geq \epsilon_{-})
\big] r.
\end{equation}
Applying the chain rule to the \ac{PRPO} objective (Eq.~\ref{eq:prpo_obj}) reveals:
\begin{equation*}
\nabla_{\!\pi} \hat{U}_{\text{PRPO}}(\pi) \!= 
	\!\!\!\!\!\sum_{q,d \in \mathcal{D}}
	\underbrace{\!\!\!\!
	\Big[\nabla_{\!\pi} \frac{\omega(d | q)}{\omega_{0}(d | q)} \Big]
	}_\text{\hspace{-1cm}grad. for single doc.\hspace{-1cm}}
	\underbrace{\!\!
	\nabla_{\!\pi} f\bigg( \! \frac{\omega(d | q)}{\omega_{0}(d | q)}, \epsilon_{-}, \epsilon_{+}, r(d | q) \! \bigg)
	}_\text{\vphantom{g}indicator reward function}
	.
\end{equation*}
Thus, we see that the gradient of \ac{PRPO} simply takes the importance weighted metric gradient per document, and multiplies it with the indicator function and reward.
As a result, \ac{PRPO} is simple to combine with existing \ac{LTR} algorithms, especially, \ac{LTR} methods that use policy-gradients~\cite{williams1992simple}, such as PL-Rank~\citep{oosterhuis2021computationally, oosterhuis2022learning} or StochasticRank~\citep{ustimenko2020stochasticrank}.
For methods in the family of LambdaRank~\citep{wang2018lambdaloss, burges2010ranknet, burges2006learning}, it is a matter of replacing the $|\Delta DCG|$ term with an equivalent for the PRPO bounded metric.

Lastly, we note that whilst we introduced \ac{PRPO} for \ac{DR} estimation, it can be extended to virtually any relevance estimation by choosing a different $r$;
e.g., one can easily adapt it for \ac{IPS}~\citep{oosterhuis2021unifying,joachims2017unbiased}, or relevance estimates from a click model~\citep{chuklin-click-2015}, etc.
In this sense, we argue \ac{PRPO} can be seen as a framework for robust safety in \ac{LTR}.

{\renewcommand{\arraystretch}{0.01}
\setlength{\tabcolsep}{0.04cm}
\begin{figure*}[ht!]
\vspace{-\baselineskip}
\centering
\begin{tabular}{c r r r }
&
 \multicolumn{1}{c}{ \small \hspace{0.5cm} Yahoo! Webscope}
&
 \multicolumn{1}{c}{ \small \hspace{0.5cm} MSLR-WEB30k}
&
 \multicolumn{1}{c}{ \small \hspace{0.5cm} Istella}
\\
\rotatebox[origin=lt]{90}{\hspace{1.4cm}\small  NDCG@5} &
\includegraphics[scale=0.465]{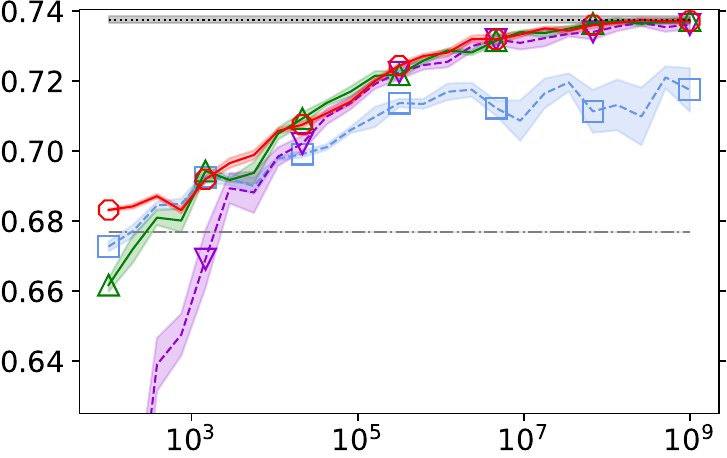} &
\includegraphics[scale=0.465]{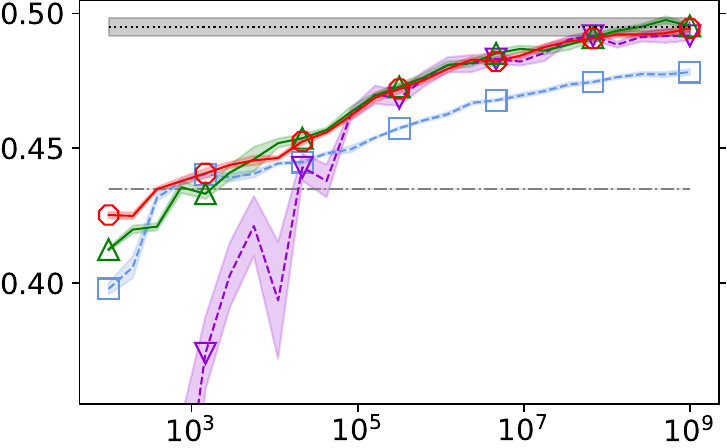} &
\includegraphics[scale=0.465]{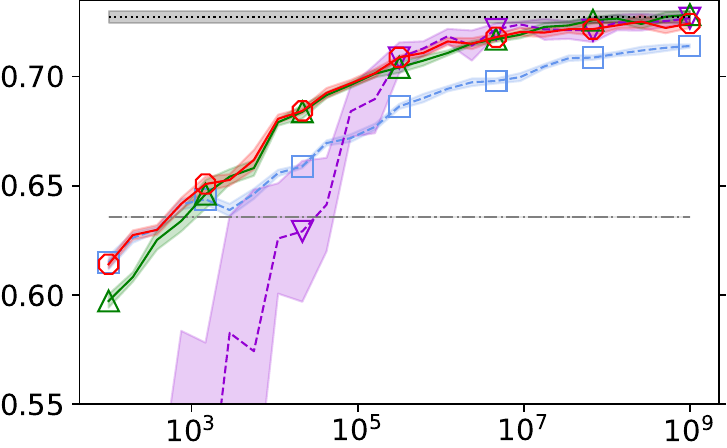}
\\
& \multicolumn{1}{c}{\small \hspace{1.75em} Number of interactions simulated ($N$)}
& \multicolumn{1}{c}{\small \hspace{1.75em} Number of interactions simulated ($N$)}
& \multicolumn{1}{c}{\small \hspace{1.75em} Number of interactions simulated ($N$)}
\\[2mm]
    \multicolumn{4}{c}{
    \includegraphics[scale=0.44]{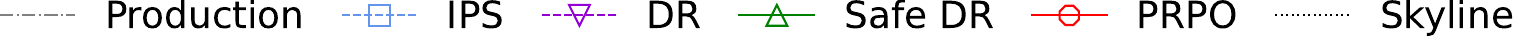}
}
\end{tabular}
%\vspace{0.1\baselineskip}
\caption{
Performance in terms of NDCG@5 of the \ac{IPS}, \ac{DR} and proposed safe \ac{DR} ($\delta=0.95$) and \ac{PRPO} ($\delta(N)=\frac{100}{N}$) methods for \ac{CLTR}.
The results are presented varying size of training data ($N$), with number of simulated queries varying from $10^2$ to $10^9$.
Results are averaged over 10 runs; the shaded areas indicate 80\% prediction intervals. 
}
\label{fig:mainresults}
\end{figure*}
}

\setlength{\tabcolsep}{0.15em}
{\renewcommand{\arraystretch}{0.60}
\begin{figure*}[ht!]
\centering
\vspace{-0.5\baselineskip}
\begin{tabular}{c r r r r}
&
 \multicolumn{1}{c}{  Yahoo! Webscope}
&
 \multicolumn{1}{c}{  MSLR-WEB30k}
&
 \multicolumn{1}{c}{  Istella}
 &

\\
\rotatebox[origin=lt]{90}{\hspace{0.65cm} \small NDCG@5} &
\includegraphics[scale=0.465]{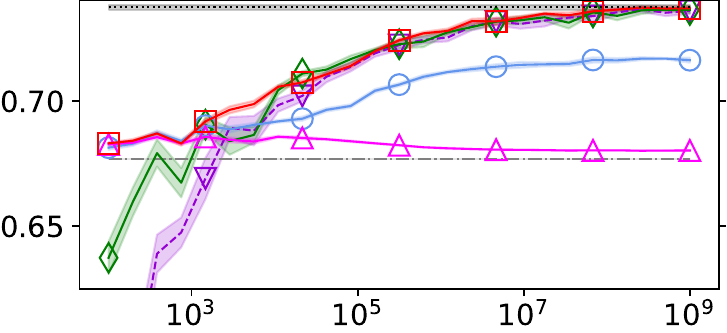} &
\includegraphics[scale=0.465]{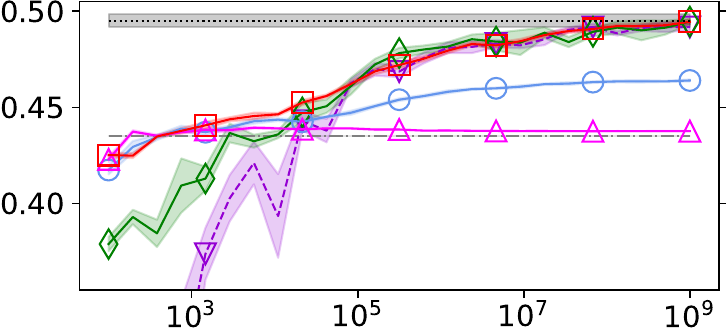} &
\includegraphics[scale=0.465]{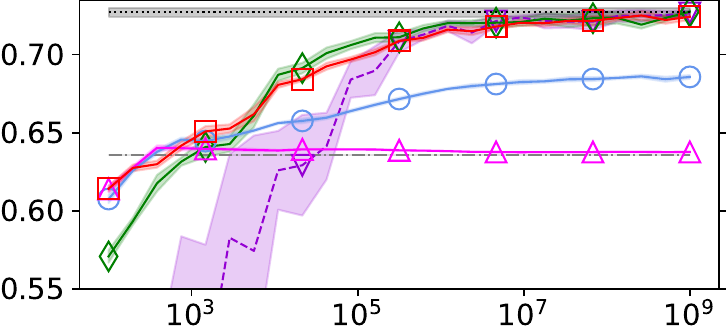} 
\\
\multicolumn{4}{c}{
\includegraphics[scale=.4]{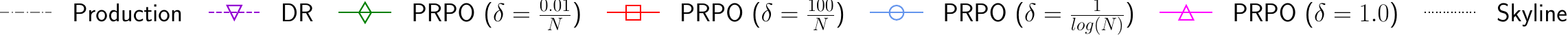}
}
\\
% \cmidrule{2-4}
%\\
% \rotatebox[origin=lt]{90}{\hspace{0.9cm}\small\it Safe \ac{DR}} 
\rotatebox[origin=lt]{90}{\hspace{0.65cm} \small NDCG@5} &
\includegraphics[scale=0.465]{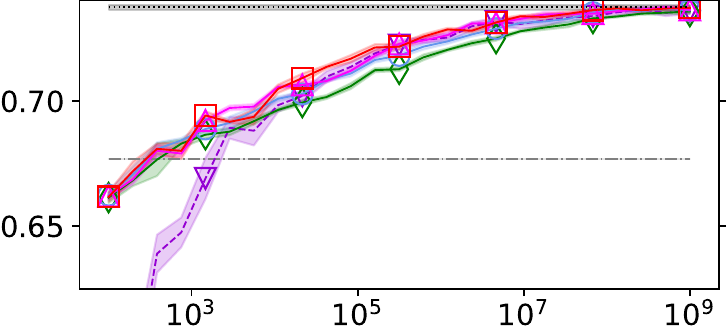} &
\includegraphics[scale=0.465]{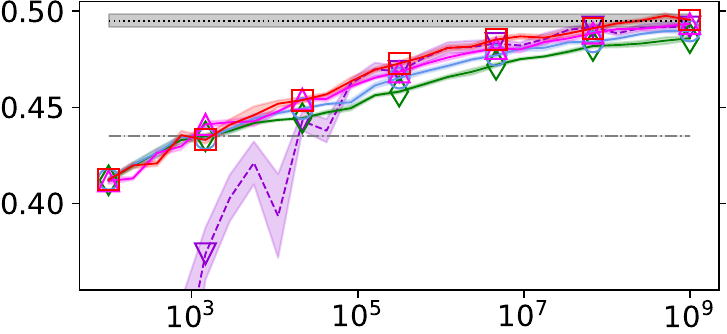} &
\includegraphics[scale=0.465]{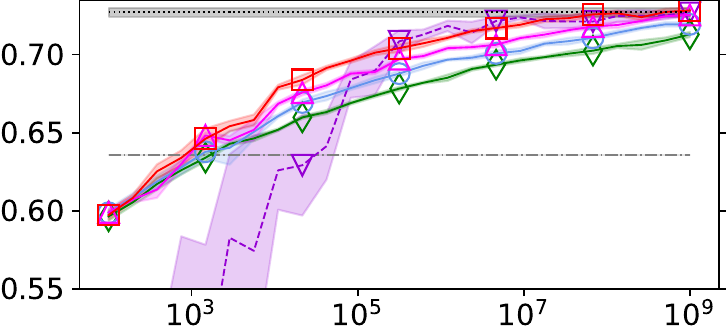}
\\
& \multicolumn{1}{c}{\small \hspace{1.75em} Number of interactions simulated ($N$)}
& \multicolumn{1}{c}{\small \hspace{1.75em} Number of interactions simulated ($N$)}
& \multicolumn{1}{c}{\small \hspace{1.75em} Number of interactions simulated ($N$)}
\\[2mm]
\multicolumn{4}{c}{
\includegraphics[scale=.4]{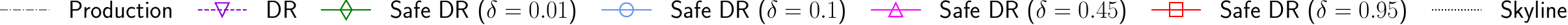}
}
\end{tabular}
%\vspace{0.1\baselineskip}
\caption{
    Performance of the safe \ac{DR} and \ac{PRPO} with varying safety parameter ($\delta$). 
    Top row: sensitivity analysis of \ac{PRPO} with varying clipping parameter ($\delta$) over varying dataset sizes $N$. 
    Bottom row: sensitivity analysis for the safe \ac{DR} method with varying safety confidence parameter ($\delta$). Results are averaged over 10 runs; shaded areas indicate $80\%$ prediction intervals.
}
\label{fig:ablationresults}
\end{figure*}
}

% !TEX root = ../2024-crm-ppo.tex

% \setlength{\tabcolsep}{0.15em}

\section{Experimental Setup}

For our experiments, we follow the semi-synthetic experimental setup that is prevalent in the \ac{CLTR} literature~\citep{oosterhuis2022doubly,oosterhuis2021robust,vardasbi2020inverse,gupta2023safe}.
We make use of the three largest publicly available \ac{LTR} datasets: Yahoo!\ Webscope~\cite{chapelle2011yahoo}, MSLR-WEB30k~\citep{qin2013introducing}, and Istella~\citep{dato2016fast}.
The datasets consist of queries, a preselected list of documents per query, query-document feature vectors, and manually-graded relevance judgments for each query-document pair.

Following previous work~\cite{oosterhuis2022doubly,vardasbi2020inverse,gupta2023safe}, we train a production ranker on a $3\%$ fraction of the training queries and their corresponding relevance judgments.
The goal is to simulate a real-world setting where a ranker trained on manual judgments is deployed in production and is used to collect click logs.
The collected click logs can then be used for \ac{LTR}. 
We assume the production ranker is safe, given that it would serve live traffic in a real-world setup. 

We simulate a top-$K$ ranking setup~\cite{oosterhuis2020policy} where only $K=5$ documents are displayed to the user for a given query, and any document beyond that gets zero exposure.
To get the relevance probability, we apply the following transformation: $P(R=1 \mid q, d) = 0.25 * rel(q,d)$, where $rel(q,d) \in \{0,1,2,3,4\}$ is the relevance judgment for the given query-document pair.
We generate clicks based on the trust bias click model (Assumption~\ref{assumption:trustbias}):
\begin{equation}
    P(C=1 \mid q, d, k) = \alpha_k P(R=1 \mid q, d) + \beta_k.   
\label{click_simulation}
\end{equation}
The trust bias parameters are set based on the empirical observation by \citet{agarwal2019addressing}: $\alpha = [0.35, 0.53, 0.55, 0.54, 0.52]$, and $\beta = [0.65, 0.26, 0.15, 0.11, 0.08]$.
For \ac{CLTR} training, we only use the training and validation clicks generated via the click simulation process (Eq.~\ref{click_simulation}).
Further, to test the robustness of the safe \ac{CLTR} methods in a setting where the click model assumptions do not hold,
we simulate an \emph{adversarial click model}. In the adversarial click model, the user clicks on the irrelevant document with a high probability and on a relevant document with a low click probability.
We define the adversarial click model mathematically as follows:
\begin{equation}
    P(C=1 \mid q, d, k) = 1 - \mleft( \alpha_k P(R=1 \mid q, d) + \beta_k \mright).   
\label{click_simulation_adv}
\end{equation}
Thereby, we simulate a maximally \emph{adversarial} user who clicks on documents with a click probability that is inversely correlated with the assumed trust bias model (Assumption~\ref{assumption:trustbias}).

Further, we assume that the logging propensities have to be estimated.
For the logging propensities $\rho_0$, and the logging metric weights ($\omega_0$), we use a simple Monte-Carlo estimate~\citep{gupta2023safe}:
\begin{equation}
    \hat{\rho}_0(d ) = \frac{1}{N} \sum^N_{i=1: y_i \sim \pi_{0}\hspace{-1cm}}   \alpha_{k_{i}(d)},
    \quad
    \hat{\omega}_0(d ) = \frac{1}{N} \sum^N_{i=1: y_i \sim \pi_{0}\hspace{-0.8cm}} \mleft( \alpha_{k_{i}(d)} + \beta_{k_{i}(d)} \mright).
    \label{prop-estimate}  
\end{equation}
%
% As is common in \ac{CLTR}~\citep{oosterhuis2020learning, saito2021counterfactual, joachims2017unbiased},    
For the learned policies ($\pi$), we optimize \ac{PL} ranking models~\citep{oosterhuis2021computationally} using the REINFORCE policy-gradient method~\cite{yadav2021policy,gupta2023safe}.
We perform clipping on the logging propensities (Eq.~\ref{policy-aware-exposure}) only for the training clicks and not for the validation set. 
Following previous work, we set the clipping parameter to $10 / \sqrt{N}$~\cite{gupta2023safe,oosterhuis2021unifying}. 
We do not apply the clipping operation for the logging metric weights (Eq.~\ref{eq:omega_logging}).
To prevent overfitting, we apply early stopping based on the validation clicks.
For variance reduction, we follow~\cite{yadav2021policy,gupta2023safe} and use the average reward per query as a control-variate.

As our evaluation metric, we compute the NDCG@5 metric using the relevance judgments on the test split of each dataset~\citep{jarvelin2002cumulated}. 
Finally, the following methods are included in our comparisons:
\begin{enumerate}[leftmargin=*]
    \item  \emph{IPS}. The \ac{IPS} estimator with affine correction~\cite{vardasbi2020inverse,oosterhuis2021unifying} for \ac{CLTR} with trust bias (Eq.~\ref{cltr-obj-ips-positionbias}).
    \item  \emph{Doubly Robust}. The \ac{DR} estimator for \ac{CLTR} with trust bias (Eq.~\ref{cltr-obj-dr}). 
    This is the most important baseline for this work, given that the \ac{DR} estimator is the state-of-the-art \ac{CLTR} method~\cite{oosterhuis2022doubly}.
    \item  \emph{Safe \ac{DR}}. Our proposed safe \ac{DR} \ac{CLTR} method (Eq.~\ref{dr-objgenbound}), which relies on the trust bias assumption (Assumption~\ref{assumption:trustbias}).
     \item  \emph{\ac{PRPO}}. Our proposed \acfi{PRPO} method for safe \ac{DR} \ac{CLTR} (Eq.~\ref{eq:prpo_obj}).
     \item  \emph{Skyline.} \ac{LTR} method trained on the true relevance labels. Given that it is trained on the real relevance signal, the skyline performance is the upper bound on any \ac{CLTR} methods performance. 
\end{enumerate}

% !TEX root = ../2024-crm-ppo.tex

\section{Results and Discussion}
\header{Comparision with baseline methods}
Figure~\ref{fig:mainresults} presents the main results with different \ac{CLTR} estimators with varying amounts of simulated click data.
Amongst the baselines, we see that the \ac{DR} estimator converges to the skyline much faster than the \ac{IPS} estimator. 
The \ac{IPS} estimator fails to reach the optimal performance even after training on $10^9$ clicks, suggesting that it suffers from a high-variance problem. 
This aligns with the findings of \citet{oosterhuis2022doubly}, who reports similar observations. 
In terms of safety, we note that when the click data is limited ($N < 10^5$), the \ac{DR} estimator performs much worse than the logging policy, i.e., it exhibits unsafe behavior, which can lead to a negative user experience if deployed online. 
A likely explanation is that when the click data is limited, the regression estimates ($\hat{R}(d)$, Eq.~\ref{cltr-obj-dr}) have high errors, resulting in a large performance degradation, compared to \ac{IPS}.

Our proposed safety methods, safe \ac{DR} and \ac{PRPO}, reach the performance of the logging policy within $\sim$500 queries on all datasets. 
For the safe \ac{DR} method, we set the confidence parameter $\delta=0.95$. For the \ac{PRPO} method, we set $\delta(N)=\frac{100}{N}$.  
On the MSLR and the ISTELLA dataset, we see that \ac{PRPO} reaches logging policy performance with almost $10^3$ fewer queries than the \ac{DR} method.  
Thus, our results demonstrate that our proposed methods, safe \ac{DR} and \ac{PRPO}, can be safely deployed, and avoid the initial period of bad performance of \ac{DR}, whilst providing the same state-of-the-art performance at convergence.

% Additionally, we note that \ac{PRPO} recovers the performance of the safe logging policy at a slightly faster rate than the safe \ac{DR} method. 
% We suspect this is because \ac{PRPO} explicitly constrains the \ac{DR} objective via the clipping operation.
% In contrast, the safe \ac{DR} method has no such explicit constraint, resulting in slightly poor performance of safe \ac{DR} as compared to the \ac{PRPO} method. 
% \harrie{I have no idea what you mean? both constrain, one through clipping the other through regularization?}

\setlength{\tabcolsep}{0.15em}
{\renewcommand{\arraystretch}{0.63}
\begin{figure*}[th!]
\vspace{-\baselineskip}
\centering
\begin{tabular}{c r r r r}
&
 \multicolumn{1}{c}{  Yahoo! Webscope}
&
 \multicolumn{1}{c}{  MSLR-WEB30k}
&
 \multicolumn{1}{c}{  Istella}
 &

\\
\rotatebox[origin=lt]{90}{\hspace{0.65cm} \small NDCG@5} &
\includegraphics[scale=0.465]{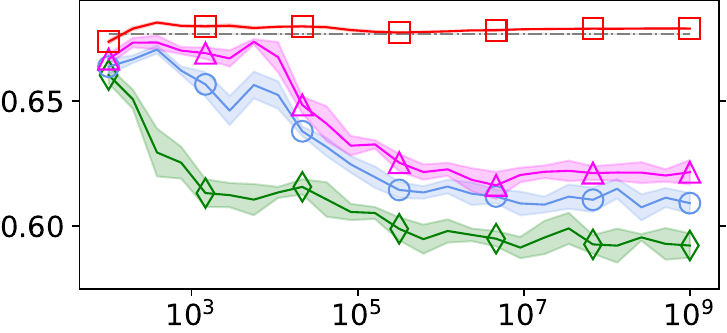} &
\includegraphics[scale=0.465]{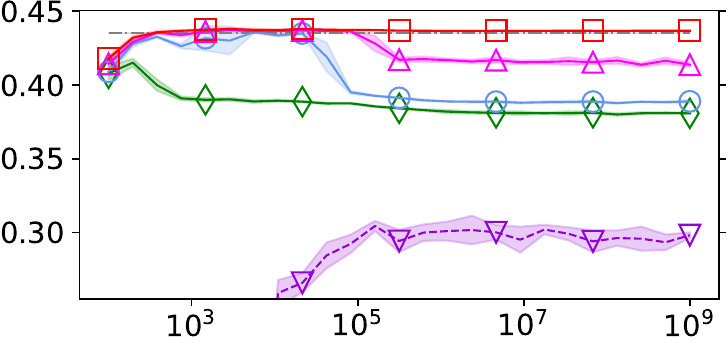} &
\includegraphics[scale=0.465]{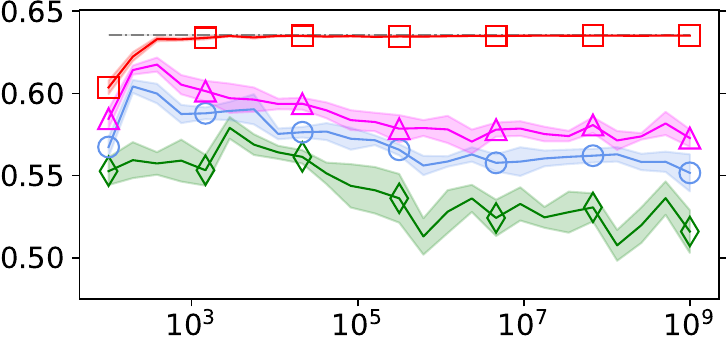} 
\\
\multicolumn{4}{c}{
\includegraphics[scale=.4]{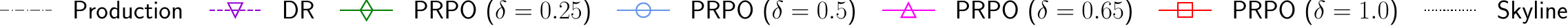}
}
\\
%\cmidrule{2-4}
%\\
\rotatebox[origin=lt]{90}{\hspace{0.65cm} \small NDCG@5} &
\includegraphics[scale=0.465]{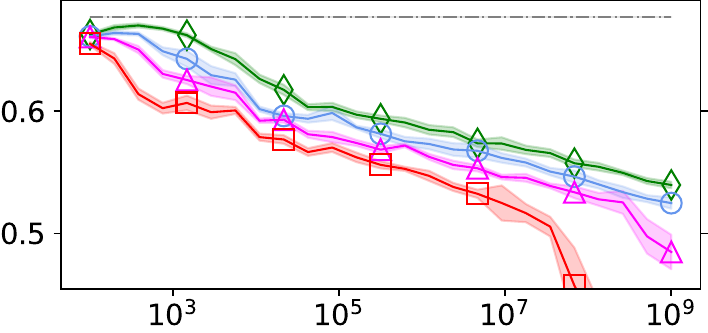} &
\includegraphics[scale=0.465]{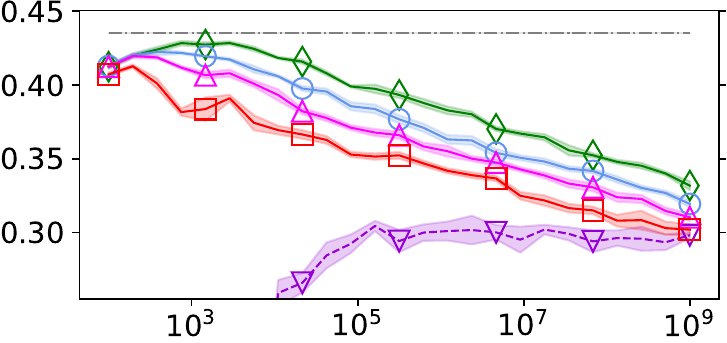} &
\includegraphics[scale=0.465]{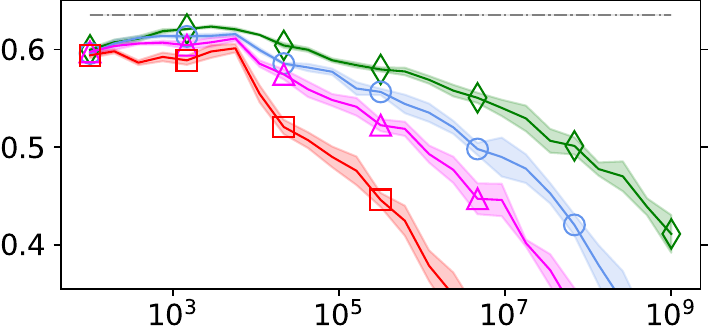}
\\
& \multicolumn{1}{c}{\small \hspace{1.75em} Number of interactions simulated ($N$)}
& \multicolumn{1}{c}{\small \hspace{1.75em} Number of interactions simulated ($N$)}
& \multicolumn{1}{c}{\small \hspace{1.75em} Number of interactions simulated ($N$)}
\\[2mm]
\multicolumn{4}{c}{
\includegraphics[scale=.4]{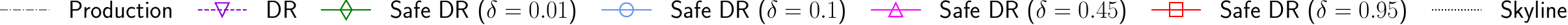}
}
\end{tabular}
%\vspace{0.1\baselineskip}
\caption{
    Performance of the proposed safe \ac{DR} and \ac{PRPO} with the adversarial click model. 
    Top: sensitivity analysis results for the \ac{PRPO} method with varying clipping parameter ($\delta$). 
    Bottom: sensitivity analysis for the safe \ac{DR} method with varying safety confidence parameter ($\delta$). 
    Results are averaged over 10 independent runs; the shaded areas indicate $80\%$ prediction intervals.
}
\label{fig:ablationresults_adv}
\end{figure*}
}

\header{Sensitivity analysis of the safety parameter}
To understand the tradeoff between safety and utility, we performed a sensitivity analysis by varying the safety parameter ($\delta$) for the safe \ac{DR} method and \ac{PRPO}. 
The top row of Figure~\ref{fig:ablationresults} shows us the performance of the \ac{PRPO} method with different choices of the clipping parameter $\delta$ as a function of dataset size ($N$). 
We report results with the setting of the $\delta$ parameter, which results in different clipping widths. 
For the setting $\delta=\frac{0.01}{N}$ and $\delta = \frac{100}{N}$, the clipping range width grows linearly with the dataset size $N$. Hence, the resulting policy is safer at the start 
but converges to the \ac{DR} estimator when $N$ increases. 
With $\delta=\frac{0.01}{N}$, the clipping range is wider at the start. As a result, it is more unsafe than when $\delta=\frac{100}{N}$, which is the safest amongst all. 
For the case where the range grows logarithmically ($\delta=\frac{1}{\log(N)}$), the method is more conservative throughout, i.e., it is closer to the logging policy since the clipping window grows only logarithmically with $N$.
For the extreme case where the clipping range is a constant ($\delta=1$), \ac{PRPO} avoids any change with respect to the logging policy, and as a result, it sticks closely to the logging policy. 

The botton row of Figure~\ref{fig:ablationresults} shows the performance of the safe \ac{DR} method with varying confidence parameter values ($\delta$). 
Due to the nature of the generalization bound (Eq.~\ref{dr-objgenbound}), the confidence parameter is restricted to: $0 \leq \delta \leq 1$.
We vary the confidence parameters in the range $\delta \in \{0.01,0.1,0.45,0.95\}$.
We note that a lower $\delta$ value results in higher safety, and vice-versa. 
Until $N < 10^5$, there is no noticeable difference in performance. 
For the Yahoo!\ Webscope dataset, almost all settings result in a similar performance. 
For the MSLR and ISTELLA datasets, when $N < 10^5$, a lower $\delta$ value results in a more conservative policy, i.e., a policy closer to the logging policy. 
However, the performance difference with different setups is less drastic than with the \ac{PRPO} method. 
Thus, we note that the safe \ac{DR} method is \emph{less flexible} in comparison to \ac{PRPO}.

Therefore, compared to our safe \ac{DR} method, we conclude that our \ac{PRPO} method provides practitioners with greater flexibility and control when deciding between safety and utility. 

\header{Robustness analysis using an adversarial click model}
To verify our initial claim that our proposed \ac{PRPO} method provides safety guarantees \emph{unconditionally}, we report results with clicks simulated via the adversarial click model (Eq.~\ref{click_simulation_adv}). 
With the adversarial click setup, the initial user behavior assumptions (Assumption~\ref{assumption:trustbias}) \emph{do not hold}. 
The top row of Figure~\ref{fig:ablationresults_adv} shows the performance of the \ac{PRPO} method with different safety parameters when applied to the data collected via the adversarial click model. 
We vary the $\delta$ parameter for \ac{PRPO} in the range $\{0.25,0.5,0.65,1.0\}$, e.g.,  $\delta=0.5$ results in $\epsilon_{-}=0.5$ and $\epsilon_{+} = 2$.
With the constant clipping range ($\delta=1$), we notice that after $\sim$400 queries, the \ac{PRPO} methods performance never drops below the safe logging policy performance. 
For greater values of $\delta$, there are drops in performance but they are all bounded. 
For the Yahoo! Webscope dataset, the maximum drop in the performance is $\sim$12$\%$; for the MSLR30K dataset, the maximum performance drop is $\sim$10$\%$; and finally, for the Istella dataset, the maximum drop is $\sim$20$\%$.
Clearly, these observations show that \ac{PRPO} provides robust safety guarantees, that are reliable even when user behavior assumptions are wrong.

In contrast, the generalization bound of our safe \ac{DR} method (Theorem~\ref{CLTR-bound}) holds only when the user behavior assumptions are true. 
%If the user behavior assumptions do not hold, the generalization bound does not hold theoretically. 
This is not the case in the bottom row of Figure~\ref{fig:ablationresults_adv}, which shows the performance of the safe \ac{DR} method under the adversarial click model. 
%The safe \ac{DR} method fails to achieve safety when the underlying assumptions break. 
Even with the setting where the safety parameters have a high weight ($\delta=0.01$), as the click data size increases, the performance drops drastically. 
Regardless of the exact choice of $\delta$, the effect of the regularization of safe \ac{DR} disappears as $N$ grows, thus in this adversarial setting, it is only a matter of time before the performance of safe \ac{DR} degrades dramatically.
% \vspace{-11.0mm}

% !TEX root = ../2024-crm-ppo.tex
% \vspace{-5.5mm}
\section{Conclusion}
In this paper, we have introduced the first safe \ac{CLTR} method that uses state-of-the-art \ac{DR} estimation and corrects trust bias.
This is a significant extension of the existing safety method for \ac{CLTR} that was restricted to position bias and \ac{IPS} estimation.
However, in spite of the importance of this extended safe \ac{CLTR} approach, it heavily relies on user behavior assumptions.
We argue that this means it only provides a \emph{conditional} concept of safety, that may not apply to real-world settings.
To address this limitation, we have made a second contribution: the \acfi{PRPO} method.
\ac{PRPO} is the first \ac{LTR} method that provides \emph{unconditional} safety, that is applicable regardless of user behavior.
It does so by removing incentives to stray too far away from a safe ranking policy.
Our experimental results show that even in the extreme case of adversarial user behavior  \ac{PRPO} results in safe ranking behavior, unlike existing safe \ac{CLTR} approaches.

\ac{PRPO} easily works with existing \ac{LTR} algorithms and relevance estimation techniques.
We believe it provides a flexible and generic framework that enables practitioners to apply the state-of-the-art \ac{CLTR} method with strong and robust safety guarantees. 
Future work may apply the proposed safety methods to exposure-based ranking fairness~\cite{oosterhuis2021computationally,yadav2021policy} and to safe online \ac{LTR}~\cite{oosterhuis2021unifying}.

\appendix
% !TEX root = ../2024-crm-ppo.tex

\section{Appendix: Extended Safety Proof}

\begin{lemma}
\label{cov-lemma}
Under the trust bias click model (Assumption~\ref{assumption:trustbias}), and
    given the trust bias parameter $\alpha_k, \beta_k$, the regression model estimates $\hat{R}_d$ and click indicator $c(d)$, the following holds:
    \begin{equation}
    \mathrm{Cov}_{y,c}\mleft[ c(d)  -  \beta_{k(d)}, \alpha_{k(d)}\hat{R}_d \mright] \geq 0.
    \end{equation}
\end{lemma}
\begin{proof}
\vspace{-0.75\baselineskip}
    The covariance term can be rewritten as:
    \begin{align}
        &\mathrm{Cov}_{y,c}\mleft[ c(d) -  \beta_{k(d)}, \alpha_{k(d)}\hat{R}_d \mright] \nonumber \\
         &= \mathbb{E}_{y,c} \mleft[ (c(d)  -  \beta_{k(d)}) \alpha_{k(d)}\hat{R}_d  \mright] - \mathbb{E}_{y,c} \mleft[ c(d)  -  \beta_{k(d)} \mright] \mathbb{E}_{y} \mleft[ \alpha_{k(d)}\hat{R}_d  \mright] \nonumber \\    
         &= \hat{R}_d \Big(\mathbb{E}_{y,c} \mleft[ c(d) \alpha_{k(d)} \mright]   - \mathbb{E}_{y} \mleft[ \beta_{k(d)} \alpha_{k(d)}  \mright] - R_d \; \rho_{0}(d)^2  \Big),
         \label{cov-expand}
    \end{align}
    where use $\rho_{0}(d)=\mathbb{E}_{y,c} \mleft[ \alpha_{k(d)}  \mright]$ and $\mathbb{E}_{y,c}\mleft[ (c_i(d) - \beta_{k_i(d)} )/\rho_{0}(d) \mright]  = R_d$~\cite{oosterhuis2022doubly}.
    Expanding the first expectation term in the expression:
    \begin{align}
        & \underset{y,c}{\mathbb{E}} \mleft[ c(d) \alpha_{k(d)} \mright] =  \!\! \sum_{y \in \pi_0} \pi_0(y) \alpha_{k(d)} P(C=1 \mid d,y) 
        =\!\! \sum_{y \in \pi_0} \pi_0(y) \alpha_{k(d)}
        \nonumber \\[-1.2ex] &
        \cdot \mleft( \alpha_{k(d)} R_d + \beta_{k(d)} \mright) 
        = R_d \mathbb{E}_{y} \mleft[ \alpha_{k(d)}^2 \mright] + \mathbb{E}_{y} \mleft[ \alpha_{k(d)} \beta_{k(d)} \mright],
        %= \sum_{y \in \pi_0} \pi_0(y) \mleft[ \alpha_{k(d)}^2 R_d + \alpha_{k(d)} \beta_{k(d)} \mright] \nonumber \\  
    \end{align}
    where we substitute click model equation $P(C=1 \mid d,y)$ (Eq.~\ref{cltr-obj-dr}). Substituting it back in Eq.~\ref{cov-expand}, we get:
    \begin{align}
        &\mathrm{Cov}_{y,c}\mleft[ c(d) -  \beta_{k(d)}, \alpha_{k(d)}\hat{R}_d \mright] = R_d \mathbb{E}_{y} \mleft[ \alpha_{k(d)}^2 \mright] - R_d \; \mathbb{E}_{y} \mleft[ \alpha_{k(d)}  \mright]^2  \nonumber \\[-1ex]
        & R_d \Big(\mathbb{E}_{y} \mleft[ \alpha_{k(d)}^2 \mright] -  \mathbb{E}_{y} \mleft[ \alpha_{k(d)}  \mright]^2 \Big) = R_d \mathrm{Var}_{y} \mleft[ \alpha_k(d) \mright] \geq 0.    \qedhere
    \end{align}
%    This completes the proof.
\end{proof}

\subsection{Proof of Theorem~\ref{CLTR-bound}}\label{sec:perfbound-proof}
    \begin{proof}
        As per Cantelli's inequality~\cite{ghosh2002probability}, the following inequality must hold with probability $1-\delta$:
        % \begin{equation}
        %     P(\mathbb{E}[\hat{X}] \geq \hat{X} - \lambda ) \geq 1-\delta.
        %      \label{cantelli-inequality} 
        % \end{equation}
        % Building on this inequality, the following inequality must hold with probability $1-\delta$:
        %
        \begin{equation}
            U(\pi) \geq \hat{U}_{\text{DR}}(\pi) - \sqrt{ \frac{1-\delta}{\delta} \mathrm{Var}_{q,y,c}\mleft[\hat{U}_{\text{DR}}(\pi)\mright]}.
            \label{inequality} 
        \end{equation}
        Following a similar approach as previous works~\cite{gupta2023safe,wu2018variance}, we look for an upper-bound on the variance of the \ac{DR} estimator.
        From the definition of $\hat{U}_{\text{DR}}(\pi)$ (Eq.~\ref{cltr-obj-dr}), the variance of the \ac{DR} estimator can be expressed as the variance of the second term:
        \begin{equation}
        \mbox{}\hspace*{-2mm}
            \underset{y,c}{\mathrm{Var}}\mleft[\hat{U}_{\text{DR}}(\pi )\mright] \!=\!  \frac{1}{N} \underset{y,c}{\mathrm{Var}}\mleft[ \sum_{d \in D}  \frac{\omega(d)}{\rho_{0}(d)} \big(c(d) \!-\! \alpha_{k(d)}\hat{R}(d) \!-\!  \beta_{k(d)} \big) \mright]. \!\mbox{}
            \label{eq:vardecom}
        \end{equation}
        Using Assumption~\ref{assumption:trustbias} and assuming that document examinations are independent from each other~\cite{gupta2023safe}, we rewrite further:
        \begin{equation}
        \begin{split}
            & N \cdot \underset{y,c}{\mathrm{Var}}\mleft[\hat{U}_{\text{DR}}(\pi )\mright] 
            =  \sum_{d \in D_{q}}^{} \underset{y,c}{\mathrm{Var}}\mleft[  \frac{\omega(d)}{\rho_{0}(d)} \big(c(d) - \alpha_{k(d)}\hat{R}(d) -  \beta_{k(d)} \big)  \mright] \\[-1.2ex]
            & \qquad\qquad =  \sum_{d \in D_{q}}^{} \mleft( \frac{\omega(d)}{\rho_{0}(d)} \mright)^2 \underset{y,c}{\mathrm{Var}}\mleft[ c(d)  -  \beta_{k(d)} - \alpha_{k(d)}\hat{R}(d)  \mright].
            \end{split}
            \label{var-dr}
        \end{equation}
       The total variance can be split into the following:
        \begin{align}
        &\mathrm{Var}_{y,c}\mleft[ c(d)  -  \beta_{k(d)} - \alpha_{k(d)}\hat{R}_{i}(d)  \mright]  = \mathrm{Var}_{y,c}\mleft[ \alpha_{k(d)}\hat{R}(d)  \mright]   \\
        &\;\;\;\;\;\;\;\; + \mathrm{Var}_{y,c}\mleft[ c(d)  -  \beta_{k(d)} \mright] - 2 \mathrm{Cov}_{y,c}\mleft[ c(d)  -  \beta_{k(d)}, \alpha_{k(d)}\hat{R}(d) \mright]. \nonumber
        \end{align}
        Using Lemma~\ref{cov-lemma}, we upper-bound the total variance term to:
        \begin{align}
            &\mathrm{Var}_{y,c}\mleft[ c(d)  -  \beta_{k(d)} - \alpha_{k(d)}\hat{R}(d)  \mright] \nonumber \\
            & \leq \mathrm{Var}_{y,c}\mleft[ \alpha_{k(d)}\hat{R}(d)  \mright] + \mathrm{Var}_{y,c}\mleft[ c(d)  -  \beta_{k(d)} \mright]. 
        \label{var-split}
        \end{align}
        Next, we consider the two variance terms separately; with the variance of the first term following:
        % and lastly, we upper bound the result using the fact that $P(R=1 | d, q)  \leq 1$:
    \begin{equation*}
        \underset{y,c}{\mathrm{Var}}\mleft[ \alpha_{k(d)}\hat{R}(d) \mright] = \underset{y,c}{\mathrm{Var}}\mleft[ \alpha_{k(d)} \mright] \hat{R}(d)^2 
         \leq \mathbb{E}_{y,c} \mleft[ \alpha_{k(d)}^2 \mright] \leq \mathbb{E}_{y} \mleft[ \alpha_{k(d)} \mright]. 
    \end{equation*}
       where we make use of the fact that $\hat{R}_d^2 \leq 1$, and $\alpha \in [0,1] \rightarrow \alpha_k^2 \leq \alpha_k$.
       Next, we consider the second term:
        \begin{align}
        &\mathrm{Var}_{y,c}\mleft[ c(d)  -  \beta_{k(d)} \mright] \leq  \mathbb{E}_{y,c} \mleft[ \big(c(d)  -  \beta_{k(d)} \big)^2 \mright] \\
        & \quad = \mathbb{E}_{y,c} \mleft[ c(d)^2  +  \beta_{k(d)}^2 - 2 c(d) \beta_{k(d)} \mright] 
        \leq \mathbb{E}_{y,c} \mleft[ c(d) \mright]   + \mathbb{E}_{y} \mleft[ \beta_{k(d)} \mright],  \nonumber
        \end{align}
       since $c(d)^2=c(d)$, $\beta_k^2 \leq \beta_k$, and $ \mathbb{E}_{y,c} \mleft[ c(d) \beta_{k(d)} \mright] \geq 0$.
        Substituting the click probabilities with Eq.~\ref{affine-click-model}, we get:
        \begin{align}
            & \mathbb{E}_{y,c} \mleft[ c(d) \mright]   + \mathbb{E}_{y,c} [ \beta_{k(d)} ] = \mathbb{E}_{y,c}[ \alpha_{k(d)}] P(R=1 |\, d) + 2 \, \mathbb{E}_{y,c} [ \beta_{k(d)}] \nonumber \\
            &\leq  \mathbb{E}_{y} \mleft[ \alpha_{k(d)} \mright] + 2 \, \mathbb{E}_{y}[ \beta_{k(d)}],
        \end{align}
        where we use the fact that $P(R=1 \mid d) \leq 1$.
        Putting together the bounds on both parts of Eq.~\ref{var-split}, we have:
        \begin{equation}
        \mathrm{Var}_{y,c}\mleft[ c(d)  -  \beta_{k(d)} - \alpha_{k(d)}\hat{R}(d)  \mright] \leq 2 \omega_0(d),
        \end{equation}
        where $\omega_0(d) = \mathbb{E}_{y} \mleft[ \alpha_{k(d)} \mright] + \mathbb{E}_{y} \mleft[ \beta_{k(d)} \mright]$. 
        Substituting the final variance upper bound in Eq.~\ref{var-dr}, we get:
        \begin{align}
            & \underset{y,c}{\mathrm{Var}}\mleft[ \sum_{d \in D}  \frac{\omega(d)}{\rho_{0}(d)} \big(c(d) - \alpha_{k(d)}\hat{R}(d) -  \beta_{k(d)} \big) \mright]  \!\leq\!  2\!\!\! \sum_{d \in D_{q}}^{} \! \mleft( \frac{\omega(d)}{\rho_{0}(d)} \mright)^2 \!\! \omega_0(d) \nonumber \\[-2ex]
            % &=  2 \sum_{d \in D_{q}}^{} \mleft( \frac{\omega(d)}{\rho_{0}(d)} \mright)^2 \omega_0(d)  \mleft(\frac{\omega_0(d)}{\omega_0(d)}\mright)^2  \nonumber \\
            & \quad =  2 \sum_{d \in D_{q}}^{} \mleft( \frac{\omega(d)}{\omega_{0}(d)} \mright)^2 \omega_0(d)  \mleft(\frac{\omega_0(d)}{\rho_0(d)}\mright)^2,
            \label{divergence}
        \end{align}
        where we multiply and divide by $\omega_0(d)^2$ in the third step.  
        % Now, using the fact that $\omega_{0}(d) \leq 1$, and using the non-zero minimum value for the propensity $\rho_{0}(d) \geq \tau$:
    %     %
    %     \begin{align}
    %         \frac{\omega(d)}{\rho_{0}(d)} \leq  1 + \frac{\max_{k} \beta_k}{\min_{k} \alpha_k}.
    %    \end{align}
    %    %
    % \begin{align}
    %     \frac{\omega_{0}(d)}{\rho_{0}(d)} \leq  \frac{1}{\tau}.
    % \end{align}
        Finally, we make use of the fact: $\frac{\omega_{0}(d)}{\rho_{0}(d)} \leq \max_{\pi_{0}} \frac{\omega_{0}(d)}{\rho_{0}(d)} \leq 1 + \max_{k} \frac{\beta_k}{\alpha_k}$,
        %\footnote{It is easy to verify this by solving for $\max \frac{\omega(d)}{\rho_{0}(d)}$ under the constraint that $\pi$ sums to 1 across all ranks via lagrangian.},
        and
       put everything back together:
       \begin{align}
%        & \mathrm{Var}_{y,c}\mleft[ \sum_{d \in D}  \frac{\omega(d)}{\rho_{0}(d)} \big(c(d) - \alpha_{k(d)}\hat{R}(d) -  \beta_{k(d)} \big) \mright] \nonumber \\  
        % & \leq  2 \mleft( 1 + \frac{\max_{k} \beta_k}{\min_{k} \alpha_k} \mright)^2 \sum_{d \in D_{q}}^{} \mleft( \frac{\omega(d)}{\omega_{0}(d)} \mright)^2 \omega_0(d) \nonumber \\
        N \cdot \mathrm{Var}_{y,c}\mleft[\hat{U}_{\text{DR}}(\pi )\mright] & \leq  2 Z \mleft( 1 + \max_{k} \frac{ \beta_k}{\alpha_k} \mright)^2 \sum_{d \in D_{q}}^{} \mleft( \frac{\omega'(d)}{\omega_{0}'(d)} \mright)^2 \omega_0'(d) \nonumber \\[-1.2ex]
        &= 2 Z \mleft( 1 + \max_{k} \frac{ \beta_k}{\alpha_k} \mright)^2 d_2(\omega \,\Vert\, \omega_0).
    \end{align}
    where $d_2(\omega \,\Vert\, \omega_0)$ is the Renyi divergence between the normalized expected exposure $\omega'(d)$ and $\omega_{0}'(d)$ (cf.\ Eq.~\ref{eq:actionbaseddiv}).
    Substituting this into the upper-bound on variance in Eq.~\ref{inequality} completes the proof.
    \end{proof}

%\vspace*{-4.8mm}
\subsection{Proof of Theorem~\ref{PRPO-proof}}\label{sec:prpo-proof}
    \begin{proof}
    Given a logging policy ranking $y_0$, a user defined metric weight $\omega$, and non-zero $r(d \mid q)$, for the choice of the clipping parameters $\epsilon_{-} = \epsilon_{+} = 1$, 
    the ranking $y^*(\epsilon_{-},\epsilon_{+})$ that maximizes the \ac{PRPO} objective (Eq.~\ref{eq:prpo_obj}) will be the same as the logging ranking $y_0$, i.e. $y^*(\epsilon_{-},\epsilon_{+})=y_0$.
    This is trivial to prove since any change in ranking can only lead in a decrease in the clipped ratio weights, and thus, a decrease in the \ac{PRPO} objective.
    Therefore, $y^*(\epsilon_{-}=1,\epsilon_{+}=1)=y_0$ when $\epsilon_{-} = \epsilon_{+} = 1$.
    Accordingly: $| U(y_0) -  U(y^*(\epsilon_{-}=1, \epsilon_{+}=1)) | = 0$ directly implies Eq.~\ref{eq:prpotheorem}.
    This completes our proof.
\end{proof}

Whilst the above proof is performed through the extreme case where $\epsilon_{-} = \epsilon_{+} = 1$ and the optimal ranking has the same utility as the logging policy ranking,
other choices of $\epsilon_{-}$ and $\epsilon_{+}$ bound the difference in utility to a lesser degree and allow for more deviation.
As our experimental results show, the power of PRPO is that it gives practitioners direct control over this maximum deviation.

% \clearpage
\section*{Acknowledgements}
This research was supported by Huawei Finland, 
the Dutch Research Council (NWO), under project numbers VI.Veni.222.269, 024.004.022, NWA.1389.20.\-183, and KICH3.LTP.20.006, 
and 
the European Union's Horizon Europe program under grant agreement No 101070212.
This work used the Dutch national e-infrastructure with the support of the SURF Cooperative using grant no.\ EINF-8200.
All content represents the opinion of the authors, which is not necessarily shared or endorsed by their respective employers and/or sponsors.

\balance
\bibliographystyle{ACM-Reference-Format}
\bibliography{bibliography}

\end{document}